\documentclass{IEEEtran}

\usepackage{amsfonts}
\usepackage{graphicx}
\usepackage{epstopdf}
\usepackage[utf8]{inputenc}
\usepackage[T1]{fontenc}
\usepackage{hyperref}
\usepackage{url}
\usepackage{booktabs}
\usepackage{amsfonts}
\usepackage{nicefrac}
\usepackage{microtype}
\usepackage{xcolor}
\usepackage{amsmath}
\usepackage{algorithm}
\usepackage{algpseudocode}
\usepackage{multirow}
\usepackage{siunitx}
\usepackage{graphicx}
\usepackage{placeins}
\usepackage{amsthm}
\usepackage{subcaption}

\newtheorem{theorem}{Theorem}

\newtheorem{claim}{Claim}
\usepackage{cleveref}

\title{Combinatorial optimization of the coefficient of determination}

\author{Marc Harary \\
Department of Data Science \\
Dana-Farber Cancer Institute \\
Boston, MA 02215 \\
\texttt{marc@ds.dfci.harvard.edu}
}

\begin{document}

\maketitle

\begin{abstract}
Robust correlation analysis is among the most critical challenges in statistics. Herein, we develop an efficient algorithm for selecting the $k$- subset of $n$ points in the plane with the highest coefficient of determination $\left( R^2 \right)$. Drawing from combinatorial geometry, we propose a method called the \textit{quadratic sweep} that consists of two steps: (i) projectively lifting the data points into $\mathbb R^5$ and then (ii) iterating over each linearly separable $k$-subset. Its basis is that the optimal set of outliers is separable from its complement in $\mathbb R^2$ by a conic section, which, in $\mathbb R^5$, can be found by a topological sweep in $\Theta \left( n^5 \log n \right)$ time. Although key proofs of quadratic separability remain underway, we develop strong mathematical intuitions for our conjectures, then experimentally demonstrate our method's optimality over several million trials up to $n=30$ without error. Implementations in Julia and fully seeded, reproducible experiments are available at \url{https://github.com/marc-harary/QuadraticSweep}.
\end{abstract}

\section{Introduction}

Correlation analysis is among the most essential tools in all of statistics, often one of the first tests applied to bivariate data \cite{pearson1896vii, degroot2002probability, casella2024statistical, hastie2009elements}. It has also been long since recognized that many common correlation measures are not robust; they are easily distorted by the presence of outlying data points that differ widely from the underlying distribution \cite{law1986robust, huber2011robust, rousseeuw2005robust, shevlyakov2016robust}. Robust correlation analysis therefore stands as a mathematical challenge of considerable importance. In the following work, we adopt the combinatorial approach of finding the $k$-subset out of $n$ points that maximizes the coefficient of determination $\left( R^2 \right)$. Equivalently, we select the $n-k$ outlying points whose removal will result in the largest increase in the goodness of fit of the least-squares regression line.

Thus far, many of the primary methods in robust statistics involve estimators that, though minimally affected by outliers, do not exclude them altogether. First developed by \cite{hodges2011estimates}, R-estimators \cite{sidak1999theory} are widely used for robustly determining both location and scale by computing results on the ranks of the data. The Spearman correlation coefficient (SCC) \cite{spearman1987proof}, Wilcoxon estimator \cite{wilcoxon1992individual}, Kendall's tau \cite{kendall1938new}, and the quadrant coefficient \cite{gibbons2014nonparametric} fall under this category. Biweight midcorrelation \cite{mosteller1977data} weighs data points in proportion to their distance from the distribution's centroid to minimize the effect of more distant values. Location and scatter can also be robustly measured by the minimum covariance determinant (MCD) estimator \cite{hubert2010minimum}. For an extensive introduction to robust estimators of correlation, the reader is directed to \cite{shevlyakov2016robust}.

Similar to our approach of partitioning the ``inliers'' and outliers, a limited number of prior combinatorial methods have been implemented. Random sample consensus (RANSAC) \cite{fischler1981random}, perhaps the most famous example, is an iterative non-deterministic algorithm that repeatedly sub-samples data to fit model parameters. Its limitation is that it is probabilistic; it is not guaranteed to sample the optimal set of inliers. More rigorously, least trimmed squares (LTS) \cite{rousseeuw1984least} deterministically identifies the subset with the smallest sum of squared errors when fitted to its respective line. But while this approach is optimal with respect to residuals, it does not, in general, identify the subset with the best $R^2$ \cite{shevlyakov2016robust}.

Herein, we treat goodness of fit as a blackbox combinatorial objective function to develop an efficient algorithm that deterministically identifies the optimal subset. Called the \textit{quadratic sweep}, its purpose is to solve an open question in discrete geometry rather than to provide a meaningful statistical estimator \textit{per se}. Our principle insight is that the optimal subset can be separated from its complement by a conic section in two dimensions. We show that, because goodness of fit ultimately reflects the inliers' anisotropy, outliers can be conceived of as ``pulling'' elliptical distributions parallel to their minor principal components. This ``pull'' that is complementary to the eccentricity of the inlying ellipse corresponds to a decision boundary that, naturally hyperbolic, separates the inliers and their complement. When the points are projected onto a higher-dimensional manifold embedded in $\mathbb R^5$, the boundary forms a linear functional that can be found via a polynomial-time topological sweep \cite{edelsbrunner1986topologically, edelsbrunner1987algorithms}.

Our paper is structured as follows. In Section \ref{sec:methods}, we provide a formal problem statement, enumerate prerequisite principals in statistics and geometry, introduce the main algorithm, and analyze its runtime. In Section \ref{sec:sep}, we develop the geometric intuitions about the relationships between outlying and inlying points that motivate the quadratic sweep. Section \ref{sec:exp} provides thorough experimental evidence of the algorithm's correctness and compares its performance to heuristic methods like RANSAC. Our results also provide a cautionary example of combinatorial optimization for isolating inliers that emphasizes the capacity of the quadratic sweep to suggest the presence of collinearity when no such trend exists. Concluding remarks are in Section \ref{sec:concl}, in which we discuss other limitations of our method and future directions for more rigorous proofs of the algorithms developed below.

\section{Methods}
\label{sec:methods}

\subsection{Problem statement}

Suppose that we have a dataset $\mathcal D = \left\{ (X_i, Y_i) \right\}_{i=1}^n$. We put the vectors $X = \left\{ X_1, X_2, \ldots, X_n \right\}$ and $Y = \left\{ Y_1, Y_2, \ldots, Y_n \right\}$. Our variance is denoted $\sigma_X^2 = \frac{1}{n} S_{XX} - \frac{1}{n^2} S_X^2$, where $S_{XX} = \sum_{i=1}^n X_i^2$ and $S_X = \sum_{i=1}^n X_i$. Equivalent definitions hold for $Y$. Our covariance is given by $\sigma_{XY} = \frac{1}{n} S_{XY} - \frac{1}{n^2} S_X S_Y$. The correlation is denoted
\begin{equation}
r = \frac{S_{XY} - \frac{1}{n} S_X S_Y}{\sqrt{S_{XX} - \frac{1}{n} S_X^2} \sqrt{S_{YY} - \frac{1}{n} S_Y^2}}.
\end{equation}
Similarly, the coefficient of determination, in this case for simple least squares (SLS), is denoted by
\begin{equation}
    R^2 = \frac{\left( S_{XY} - \frac{1}{n} S_X S_Y \right)^2}{\left( S_{XX} - \frac{1}{n} S_X^2 \right)\left( S_{YY} - \frac{1}{n} S_Y^2 \right)}.
\end{equation}

We conceive of $\mathcal D$ as being partitioned into outliers $\mathcal O$ and ``inliers'' $\mathcal I$, where $|\mathcal O| = n-k$ and $|\mathcal I| = k$. We assume that the effect of $\mathcal O$ is to non-strictly decrease the $R^2$ of our ``true'' distribution $\mathcal I$:
\begin{equation}
\mathcal O = \max_{\mathcal S \subseteq \mathcal D, \ |\mathcal S| = n - k} R^2_{\mathcal D \setminus \mathcal S} \quad \text{and} \quad \mathcal I = \max_{\mathcal S \subseteq \mathcal D, \ |\mathcal S| =k} R^2_\mathcal S.
\label{eq:inliers}
\end{equation}
Therefore, we are interested in solving the combinatorial optimization problem suggested by \eqref{eq:inliers}:
\begin{equation}
\begin{aligned}
    \max_{\mathbf{z} \in \{0, 1\}^n} \ & \frac{\left( S_{XY} - \frac{1}{k} S_X S_Y \right)^2}{\left( S_{XX} - \frac{1}{k} S_X^2 \right) \left( S_{YY} - \frac{1}{k} S_Y^2 \right)} \\
    \text{subject to} \ & \sum_{i=1}^{n} z_i = k, \\
    & S_V = \sum_{i=1}^{n} z_i V_i, \quad V \in \{XX, XY, YY, X, Y\}.
\end{aligned}
\tag{CDO}
\label{eq:CDO}
\end{equation}

For the maximal $\mathbf z^*$, we assume that
\begin{equation}
    \mathcal I = \left\{ \left( X_i, Y_i \right) \mid z^*_i = 1\right\} \quad \text{and} \quad \mathcal O = \left\{ \left( X_i, Y_i \right) \mid z^*_i = 0\right\}.
\end{equation}
Note that this assumption may very well be untrue for real-world datasets in which an outlier $\left(X_o, Y_o \right)$ could coincidentally be more favorable with respect to $R^2$ to include in the dataset rather than an inlier $\left(X_i, Y_i \right)$. Our goal, however, is to solve the combinatorics problem \eqref{eq:CDO} independent of its real-world statistical relevance.

\subsection{Linear separability and $k$-sets}

We say that two sets of points $A$ and $B$, corresponding to matrices $\mathbf A \in \mathbb R^{n \times d}$ and $\mathbf B \in \mathbb R^{k \times d}$, are \textit{linearly separable} in $\mathbb R^d$ if there exist $\mathbf w \in \mathbb R^d$ and $b \in \mathbb R$ such that $\mathbf w^\top \mathbf A_i \geq b \quad \forall i \in [n]$ and $\mathbf w^\top \mathbf B_i \leq b \quad \forall i \in [k]$. Moreover, putting $C := A \cup B$, we say that $B$ forms a \textit{$k$-set} of $C$.

Fixing $C$, the number of $k$-sets is bounded by $O \left(n^{ \lceil d/2 \rceil} k^{ \lfloor d/2 \rfloor}\right)$ in a well-known result \cite{clarkson1988applications}. A similar finding is that a data structure can be constructed in $O \left(n^{ \lceil d/2 \rceil} k^{ \lfloor d/2 \rfloor}\right)$ time to allow us to query, in $O(\log n)$ time, the subset $B \subseteq C$ such that $\mathbf w^\top \mathbf B_i \leq b$ for any $\mathbf w \in \mathbb R^d$ and $b \in \mathbb R$. Iterating over all $k$-sets in $C$ is referred to as a \textit{topological sweep} \cite{edelsbrunner1986topologically}.

We denote the \textit{convex hull} of $\mathbf A$ as
\begin{equation}
\operatorname{conv}\left(\mathbf A \right) := \left\{ \mathbf A \boldsymbol \lambda \mid \boldsymbol \lambda \in \mathbb R^{d}, \ \sum_{i=1}^d \lambda_i = 1, \ \boldsymbol \lambda_i \geq 0 \right\}.
\end{equation}
The Hyperplane Separation Theorem \cite{de2000computational} dictates that the convex hulls of $A$ and $B$ do not intersect if and only if they are linearly separable: $\operatorname{conv}\left(\mathbf A \right) \cap \operatorname{conv}\left(\mathbf B \right) = \emptyset \Longleftrightarrow \exists \mathbf w\in \mathbb R^d, d \in \mathbb R \ \forall i \in [d] \ \mathbf w^\top \mathbf A_i \geq b \text{ and } \mathbf w^\top \mathbf B_i \leq b$. Therefore, given a partition of $C$ into $A$ and $B$ such that $|A| = k$, we may conclude that $A$ is a $k$-set of $C$ if and only if the convex hulls of $A$ and $B$ do not intersect.

\subsection{Demonstrating separability}

We will map $\mathcal I$ and $\mathcal O$ into higher dimensional vector spaces such that $\operatorname{conv}\left( \mathbf V^{(\mathcal I)} \right) \cap \operatorname{conv}\left( \mathbf V^{(\mathcal O)} \right) = \emptyset$,  where $\mathbf V^{(\mathcal I)}$ and $\mathbf V^{(\mathcal O)}$ are the two sets of points concatenated into separate matrices. Intersections of the two hulls can be detected by computing the smallest distance $d^*$ between any two points $\mathbf v_1 \in \operatorname{conv}\left( \mathbf V^{(\mathcal I)} \right)$ and $\mathbf v_2 \in \operatorname{conv}\left( \mathbf V^{(\mathcal O)} \right)$. If $d^* = 0$, then the points are identical and the intersection is non-empty. We can compute $d^*$ by solving the following quadratic program \cite{frank1956algorithm} via an interior point method \cite{wright1997primal}:
\begin{equation}
\begin{aligned}
    \min \ & \lVert \boldsymbol{\lambda}^\top \mathbf V^{(\mathcal I)} - \boldsymbol{\mu}^\top \mathbf V^{(\mathcal O)} \rVert_2^2 \\
    \text{subject to} \ & \sum_{i=1}^{d} \lambda_i = 1, \\
    & \lambda_i \geq 0 \quad \forall i \in [d], \\
    & \sum_{i=1}^{d} \mu_i = 1, \\
    & \mu_i \geq 0 \quad \forall i \in [d]. \\
\end{aligned}
\tag{HDO}
\label{eq:HDO}
\end{equation}

\subsection{Quadratic sweep}
Motivating our algorithm is the conjecture (Claim \ref{claim:sep}) that $\mathcal I$ and $\mathcal O$ are separable via a curve $\Gamma$ that is quadratic in both $X$ and $Y$. Our task is now to efficiently iterate over each separable subset $S$. Here, we draw inspiration from Delaunay triangulation \cite{preparata2012computational, edelsbrunner1987algorithms} and Cover's theorem \cite{cover1965geometrical} to linearize $\mathcal I$ and $\mathcal O$ by ``lifting'' them to a higher dimensional space via
a projective transform $\mathcal{L}_d: \mathbb{R}^2 \to \mathbb{R}^d$.

Putting $\mathbf v_{d}^{(o)} := \mathbf{\mathcal L_{d}} \left( X_o, Y_o \right)$, our polynomial $\Gamma$ corresponds to a functional $\mathbf w^*_{d} \in \mathbb R^{d}$ such that $\left( \mathbf w^{*}_{d} \right)^\top \mathbf v_{d}^{(o)} < 0 \quad \forall i \in \mathcal I, \ o \in \mathcal O$. We now reduce $\mathbf V^{(\mathcal I)}$ to a $k$-set, meaning that the search for $\mathcal I$ can be performed via a topological sweep \cite{edelsbrunner1986topologically, edelsbrunner1987algorithms} in efficient time. Fixing $\mathbf w^*$, we recognize that it can be rotated to a limited extent while still separating $\mathbf V^{(\mathcal I)}$ and $\mathbf V^{(\mathcal O)}$; during our rotation, $\mathbf w^*$ only ceases to separate the two sets once it intersects a point $\mathbf v^{(o)}$ or $\mathbf v^{(i)}$. In total, $d$ intersections uniquely determine a hyperplane in $\mathbb R^d$ \cite{preparata2012computational}; in other words, $\mathbf w^*$ can be determined ultimately from a $d$-tuple of points $P$ in $\mathbf V^{(\mathcal D)}$, where $\mathbf V^{(\mathcal D)}$ is our lifted dataset. The hyperplane with $\mathbf w_{P}$ can be computed from the kernel of the matrix $L_P$ corresponding to $P$ and augmented with the vector $\mathbf 1$ to account for an intercept term:
\begin{equation}
    N \gets \ker([L_P, \mathbf{1}]).
\end{equation}
We then obtain $\mathbf w_{P}$ by excluding the intercept:
\begin{equation}
    \mathbf w_P \gets N_{1:(d-1)}.
\end{equation}

To compute our candidate $k$-sets, we lift the the complement of $P$, project it onto $\mathbf w_P$, and compute the antiranks of their products:
\begin{align}
   C &\gets \{1, 2, \dots, n\} \setminus P \\
   \mathcal{P} &\gets \langle L_C, \mathbf w_P \rangle  \\
    \pi &\gets \Call{ArgSort}{\mathcal{P}, \rho}
\end{align}
where $L_C$ denotes the lifted complement and $\rho \in \left\{ \texttt{false}, \texttt{true} \right\}$ is a flag to signify whether sorting is performed in ascending or descending order (its significance is further explained below). While we might now be tempted only to compute the score of $S$ with this set corresponding to $\pi[1:k]$, this fails to take into consideration the fact that $S$ might also contain a subset of $P$. We must therefore iterate over every subset $P^\prime \subseteq P$, then let $S$ be $P^\prime \cup C^\prime$, where $C^\prime = \pi[1:k-|P^\prime|]$. The full algorithm, \textproc{NaiveQuadraticSweep} is written below (\ref{alg:naive}). For a flexible implementation that is applicable to a range of objectives, the score and lift functions $\mathcal S$ and $\mathcal L$ are parameters of the function. It has the following runtime:
\begin{theorem}
    \label{thm:quadratic-sweep-runtime}
\textproc{NaiveQuadraticSweep} runs in $\Theta\left(n^{d+1} \log n \right)$ time.
\end{theorem}

\begin{proof}
    We assume that the lifting procedure in line 2 requires $O(nd)$ time. The outer loop in line 7 runs for a total of $O(n^d)$ iterations. Within it, computing the kernel of $\left[ L_P, \mathbf 1 \right]$ in line 9 can be performed by a range of methods (e.g., QR decomposition \cite{householder1958unitary}, singular value decomposition \cite{eckart1936approximation}), all of which run in $O(d^3)$ time. $C$ is computable in $O(n)$ time and the matrix-vector multiplication in line 12 requires $O(nd)$ time. Sorting in line 13 has a trivial runtime of $O(n \log n)$. The inner loop at line 14 requires $d$ iterations, each of which require $O(n)$ time, meaning that the total runtime of lines 14-24 is $O(nd)$. In turn, this implies that the total run of the inner loop is $O(d^3 + n \log n + nd)$. Typically, we can expect to have $d \ll n$ such that sorting will dominate the matrix operations, giving us a total of $O(n \log n)$ for the inner loop. Therefore, our total runtime is $O(n^{d+1} \log n)$. Moreover, the outer loop in line 7 does not terminate prematurely; this gives us a runtime of $\Theta(n^{d+1} \log n)$.
\end{proof}

Notably, however, \textproc{NaiveQuadraticSweep} can be improved due to the results from \cite{clarkson1988applications}:
\begin{theorem}
    Quadratic sweeps runs in $\Theta \left(n^d \log n \right)$ time.
\end{theorem}

\begin{proof}
    Building the data structure described by \cite{clarkson1988applications} requires $O \left( n^{\lceil d/2 \rceil} k^{\lfloor d/2 \rfloor} \right)$. Each subsequent range query runs in only $O(\log n)$ time, and will still require $O \left( n^d \right)$ such queries. We therefore have a net runtime of $O(n^d \log n)$. Similar reasoning as above gives us a runtime of $\Theta(n^d \log n)$.
\end{proof}

\begin{algorithm}[h!]
\caption{Optimal Subset Selection via Lifting}
\begin{algorithmic}[1]
\Function{NaiveQuadraticSweep}{$X, Y, k, \mathcal{L}, \mathcal{S}, \rho$}
    \State $\mathcal{L}(X, Y) \to L$ 
    \State $n \gets |L|$
    \State $d \gets \dim(L)$
    \State $S^* \gets \emptyset$
    \State $\mathcal{S}^* \gets -\infty$
    
    \For{$P \subseteq \{1, 2, \dots, n\}, |P| = d$}
        \State $L_P \gets \{ L_i \mid i \in P \}$
        \State $N \gets \ker([L_P, \mathbf{1}])$ 
        \State $w \gets N_{1:(d-1)}$ 
        
        \State $C \gets \{1, 2, \dots, n\} \setminus P$
        \State $\mathcal{P} \gets \langle L_C, w \rangle$
        \State $\pi \gets \Call{ArgSort}{\mathcal{P}, \rho}$ 
        
        \For{$i \in \{1, 2, \dots, \lceil d/2 \rceil \}$}
            \For{$P' \subseteq P, |P'| = i$}
                \For{$C' \subseteq \pi[1:k-i]$}
                    \State $S \gets P' \cup C'$
                    \If{$\mathcal{S}(X_S, Y_S) > \mathcal{S}^*$}
                        \State $\mathcal{S}^* \gets \mathcal{S}(X_S, Y_S)$
                        \State $S^* \gets S$
                    \EndIf
                \EndFor
            \EndFor
        \EndFor
    \EndFor
    
    \State \Return $S^*$
\EndFunction
\end{algorithmic}
\label{alg:naive}
\end{algorithm}

\section{Quadratic separability}
\label{sec:sep}

In the following section, we progressively build towards an intuition of the quadratic separability of $\mathcal I$ and $\mathcal O$ by studying closely related objective functions. Though we only provide proofs for univariate variance---the simplest of our objectives---we suspect that one path to more rigorous solutions lies in matroid theory \cite{oxley2006matroid}. Creating level sets of each objective function, we hypothesize that filtrations \cite{oxley2006matroid} may be helpful in proving that inliers and outliers remain separable when considering various subsets of the data.

Furthermore, it is critical to note that of the following optimization problems, some involve minimization rather than maximization. In these cases, inliers lie outside the boundary $\Gamma$ rather than inside it. Sorting the complement $C$ of each hyperplane corresponding to a $d$-tuple $P$ therefore requires ascending rather than descending order. As a consequence, the flag $\rho$ must be able to be set either \texttt{true} or \texttt{false} for a software implementation of \textproc{NaiveQuadraticSweep}, hence its inclusion as a parameter in our pseudocode.

\subsection{Univariate variance}

As a warm-up, we consider solving the far simpler optimization problem  that consists of minimizing $\sigma^2_X$. This is equivalent to minimizing the sum of squared deviations:
\begin{equation}
\begin{aligned}
    \min_{\mathbf{z} \in \{0, 1\}^n} \ & S_{XX} - \frac{1}{k} S_X^2 \\
    \text{subject to} \ & \sum_{i=1}^{n} z_i = k, \\
    & S_V = \sum_{i=1}^{n} z_i V_i, \quad V \in \{XX, X\}.
\end{aligned}
\tag{VO}
\label{eq:VO}
\end{equation}

On one hand, its solution is trivially computable in $O( n \log n)$ by simply sorting $X$. The antiranks $\pi$ can then be searched for a set of $k$ adjacent points in a loop requiring $O(n)$ time to find $X_{\pi\left[\ell:\ell+k \right]}$. This is due to the fact that no outlier $o$ can lie in between two inliers $i_1$ and $i_2$ on the continuum; $\mathcal I$ must, in other words, be tightly ``packed'' with no outliers in between (see Figure \ref{fig:vo_linear}). Therefore, we can create a decision boundary that consists only of the pair
\begin{equation}
    \Gamma: \left\{ \Gamma_1, \Gamma_2 \right\}
\end{equation}
separating $\mathcal I$ and $\mathcal O$.

We can improve on the sorting approach if we consider higher-dimensional lifting. Suppose that we applied the following transform to our data:
\begin{equation}
\mathcal{L}_2: \mathbb{R} \to \mathbb{R}^2, \quad x \mapsto \mathcal{L}_2(x) = \left(x^2, x \right).
\label{eq:l2}
\end{equation}
This effectively projects the points onto the surface of a parabola, introducing curvature that we can exploit. When the data are lifted, lines now intersect our data manifold at two distinct points, enabling a single line $H$ to create our decision boundary $\Gamma$ at $\Gamma_1$ and $\Gamma_2$ simultaneously (see Figure \ref{fig:vo}). Separability can be formalized as follows.
\begin{theorem}
    In \eqref{eq:VO}, $\mathcal I$ and $\mathcal O$ are quadratically separable.
\end{theorem}

\begin{proof}
    Fix $\mathcal I$ and $\mathcal O$. We claim that there is no point $X_o$ such that $X_{i_1} \leq X_o \leq X_{i_2}$ for all $i_1, i_2 \in \mathcal I$; by way of contradiction, suppose that this is not the case. Without loss of generality, let $X_{i_1} \leq X_o \leq X_{i_2} \leq 0$. We remove $X_{i_1}$ from $\mathcal I$ to obtain a dataset $\mathcal I^\prime$, and assume, again without loss of generality, that $S^\prime_X = \sum_{i \neq i_1} X_i = 0$. If we add back $i_1$ back to $\mathcal I^\prime$, our new sum of squared deviations will be given by $SSD_{i_1} = S^\prime_{XX} + X_{i_1}^2 - \frac{1}{k} X_{i_1}^2 =S^\prime_{XX} - \frac{k-1}{k} X_{i_1}^2$. Similarly, adding $o$ to $\mathcal I^\prime$, we obtain a new sum of squared deviations given by $SSD_{o} = S^\prime_{XX} - \frac{k-1}{k} X_{o}^2$. But because $X_{i_1} \leq X_o \leq 0$, we then have $SSD_{i_2} \geq SSD_{o}$, which contradicts optimality.

    Now let $X^\prime = \max_{i \in \mathcal I} \left( \bar X_\mathcal I - X_i \right)^2$, where $\bar X_\mathcal I$ is the mean of $\mathcal I$. It follows from our previous result that there can be no $X_o \in \left[ \bar X_\mathcal I - X^\prime, X_\mathcal I + X^\prime \right]$. Put
    \begin{equation}
        \Gamma = \left\{ \left( X, X^2 \right) \mid \exists \lambda \in \left[ 0, 1 \right] \ X = \bar X_\mathcal I - X^\prime + 2 \lambda \left( X^\prime \right) ^2 \right\}.
    \end{equation}
    Due to the convexity of our lift function $\mathcal L_2$, we then know that $o \in \mathcal O$ cannot lie beneath $\Gamma$. Then $\Gamma$ quadratically separates $\mathcal O$ and $\mathcal I$, completing our proof.
\end{proof}

Additionally, our runtime can potentially improve if $k < \log n$:
\begin{theorem}
    \eqref{eq:VO} can be solved in $O\left( n (k+1)^{1/3} \right)$.
\end{theorem}

\begin{proof}
    In a well-known result due to \cite{chan1999remarks}, we recall that the number of $k$-sets in the plane given $n$ points is bounded by $O\left( n (k+1)^{1/3} \right)$. These subsets can be swept in a similar runtime by maintaining a dynamic convex hull \cite{brodal2002dynamic}. It is then trivial to maintain dynamic sums $S_{XX}$ and $S_X$ of our current $k$-set that can be subsequently used to compute the sum of squared deviations as per \eqref{eq:VO}. Updates can be performed in a manner similar to those in Welford's online algorithm \cite{welford1962note}, ensuring that the additional computational cost incurred by each iteration of the sweep remains constant. This gives us our desired runtime.
\end{proof}

\begin{figure}
    \centering
    \begin{subfigure}[b]{0.3\textwidth}
        \centering
        \includegraphics[width=\textwidth]{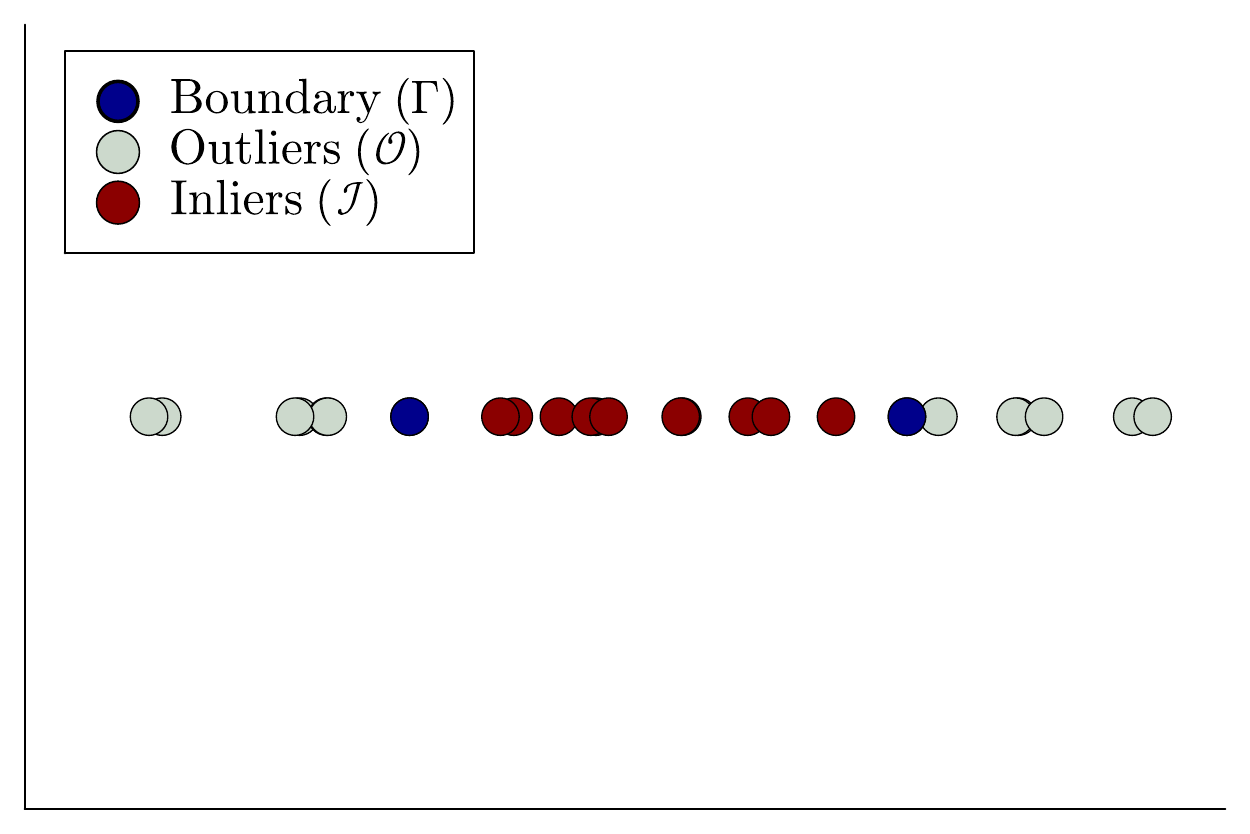}
        \caption{Unidimensional variance \eqref{eq:VO}}
        \label{fig:vo_linear}
    \end{subfigure}
    \hfill
    \begin{subfigure}[b]{0.3\textwidth}
        \centering
        \includegraphics[width=\textwidth]{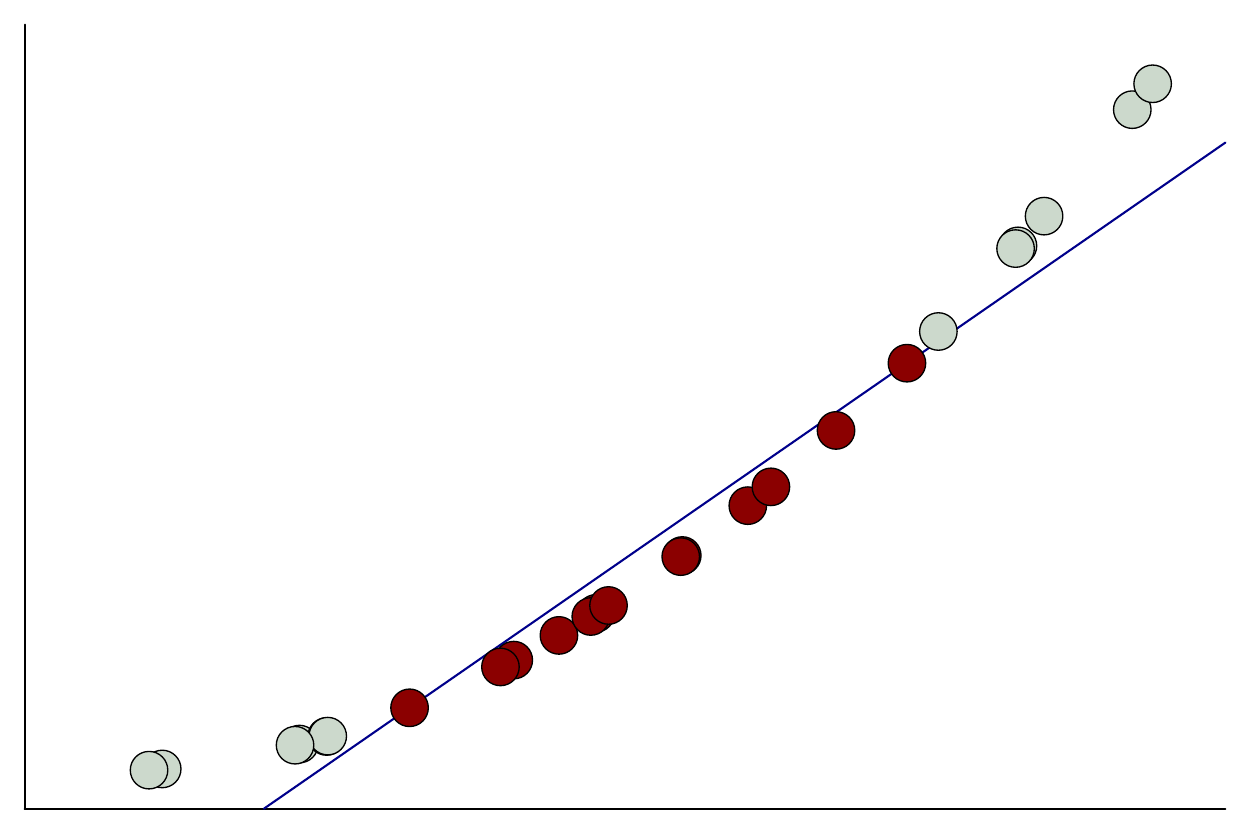}
        \caption{$X$ lifed via $\mathcal L_2$ \eqref{eq:l2}}
        \label{fig:vo}
    \end{subfigure}
    \caption{Higher-dimensional lifting from $\mathbb R$ to $\mathbb R^2$ linearizes quadratic boundaries, enabling a topological plane sweep. Our boundary is converted from two points $\Gamma_1, \Gamma_2 \in \mathbb R$ to a linear decision boundary $\Gamma < \mathbb R^2$.}
\end{figure}
\subsection{Total variation}

In two dimensions, we now consider the problem of minimizing the total variation, $\sigma_X^2 + \sigma_Y^2$:
\begin{equation}
\begin{aligned}
    \min_{\mathbf{z} \in \{0, 1\}^n} \ & \left( S_{XX} - \frac{1}{k} S_X^2 \right) + \left( S_{YY} - \frac{1}{k} S_Y^2 \right) \\
    \text{subject to} \ & \sum_{i=1}^{n} z_i = k, \\
    & S_V = \sum_{i=1}^{n} z_i V_i, \quad V \in \{XX, YY, X, Y\}.
\end{aligned}
\tag{TVO}
\label{eq:TVO}
\end{equation}
A similar principle applies, only we now create a boundary
\begin{equation}
\Gamma: \frac{\left( X - X_\circ \right)^2}{A^2} + \frac{\left( Y - Y_\circ \right)^2}{B^2} = C,
\end{equation}
which consists of an ellipse with parameters $X_\circ, Y_\circ, A, B, C \in \mathbb R$. The fact that $\Gamma$ is elliptical follows intuitively from the definition of our objective function, which is identical to that of an equation for this conic section. Our lift function now becomes
\begin{equation}
\mathcal{L}_4: \mathbb{R} \to \mathbb{R}^4, \quad x \mapsto \mathcal{L}_4(x) = \left(x^2, y^2, x, y \right).
\end{equation}
We are therefore prompted to conjecture separability as before:
\begin{claim}
    In \eqref{eq:TVO}, $\mathcal I$ and $\mathcal O$ are quadratically separable.
\end{claim}

\subsection{Difference of variances and covariance}
Third, we consider the deceptively informative task of maximizing $\sigma_X^2 - \sigma_Y^2$:
\begin{equation}
\begin{aligned}
    \max_{\mathbf{z} \in \{0, 1\}^n} \ & \left( S_{XX} - \frac{1}{k} S_X^2 \right) - \left( S_{YY} - \frac{1}{k} S_Y^2 \right) \\
   \text{subject to} \ & \sum_{i=1}^{n} z_i = k, \\
    & S_V = \sum_{i=1}^{n} z_i V_i, \quad V \in \{XX, YY, X, Y\}.
\end{aligned}
\tag{DVO}
\label{eq:DVO}
\end{equation}
In this case, we can predict that
\begin{equation}
\Gamma: \frac{\left( X - X_\circ \right)^2}{A^2} - \frac{\left( Y - Y_\circ \right)^2}{B^2} = C,
\end{equation}
a decision boundary consisting of a horizontal hyperbola at $
\left(X_\circ, Y_\circ \right)$. If we define $U = X + Y$ and $V = X - Y$, it can be shown that
\begin{equation}
    \sigma^2_U - \sigma^2_V \propto \sigma_{XY},
\end{equation}
meaning that to \eqref{eq:DVO}, we can reduce a fourth problem of maximizing $\sigma_{XY}$ with an appropriate change of variables:
\begin{equation}
\begin{aligned}
    \max_{\mathbf{z} \in \{0, 1\}^n} \ & S_{XY} - \frac{1}{k} S_X S_Y \\
    \text{subject to} \ & \sum_{i=1}^{n} z_i = k, \\
    & S_V = \sum_{i=1}^{n} z_i V_i, \quad V \in \{XY, X, Y\}.
\end{aligned}
\tag{CVO}
\label{eq:CVO}
\end{equation}
Correspondingly, our decision boundary is rotated $\pi/4$ radians to form a region bounded by a rectangular hyperbola
\begin{equation}
\Gamma: \left( X - X_\circ \right) \left( Y - Y_\circ \right) = A.
\label{eq:hyper}
\end{equation}
Retaining our lift function $\mathcal{L}_4$, we again claim separability.
\begin{claim}
    In \eqref{eq:DVO} and \eqref{eq:CVO}, $\mathcal I$ and $\mathcal O$ are quadratically separable.
\end{claim}

\subsection{Coefficient of determination}
We now return to \eqref{eq:CDO}. The trick to $R^2$ is to both consider \eqref{eq:TVO} and the principal components of $\mathcal I$. A distribution with a high $R^2$ will have a large difference $\lambda_1 - \lambda_2$---or, very similarly, $\sqrt{\lambda_1^2 - \lambda_2^2}$---indicating that the points are far more ``spread'' along the first principal component $\mathbf v_1$ than the second component $\mathbf v_2$. Geometrically, the difference corresponds to the ``eccentricity'' $e$ of the distributions (for parametric distributions that are inherently elliptical in nature, our analogy is literal). Very similarly, the ratio $\lambda_1 / \lambda_2$ gives the anisotropy of our data, a geometrically similar property.

Our task of finding $\mathcal I$ thus corresponds to finding the subset that is most eccentric, or ``spread'' in some direction $\mathbf v_1$. Conversely, if we project our data onto its principal components, we rotate and shear $\mathbf v_1$ and $\mathbf v_2$ such that they now point in the direction of the $Y$- and $X$-axes, respectively. If our data are whitened, our subset $\mathcal I$ now corresponds to the subset that is most ``vertically stretched,'' a problem strongly reminiscent of \eqref{eq:DVO}. We therefore draw parallels to \eqref{eq:hyper}

Returning to our original vector space and inverting the whitening transform, $\Gamma$ becomes a sheared hyperbola no longer in standard or rectangular form. It is instead given by a more general conic section
\begin{equation}
    \Gamma: A \left( X - X_\circ \right)^2 + B \left( Y - Y_\circ \right)^2 + C \left( X - X_\circ \right) \left( Y - Y_\circ \right) =D.
    \label{eq:hyper_bound}
\end{equation}
Effectively, we have introduced two location parameters and three scale parameters, necessitating a more complex lift function:
\begin{equation}
    \mathcal{L}_5: \mathbb{R} \to \mathbb{R}^5, \quad x \mapsto \mathcal{L}_5(x) = \left(x^2, xy, y^2, x, y \right).
\end{equation}
Although our final proof remains underway, we claim the following:
\begin{claim}
    In \eqref{eq:CDO}, $\mathcal I$ and $\mathcal O$ are quadratically separable.
    \label{claim:sep}
\end{claim}

\begin{figure}
    \centering
    \begin{subfigure}[b]{0.3\textwidth}
        \centering
        \includegraphics[width=\textwidth]{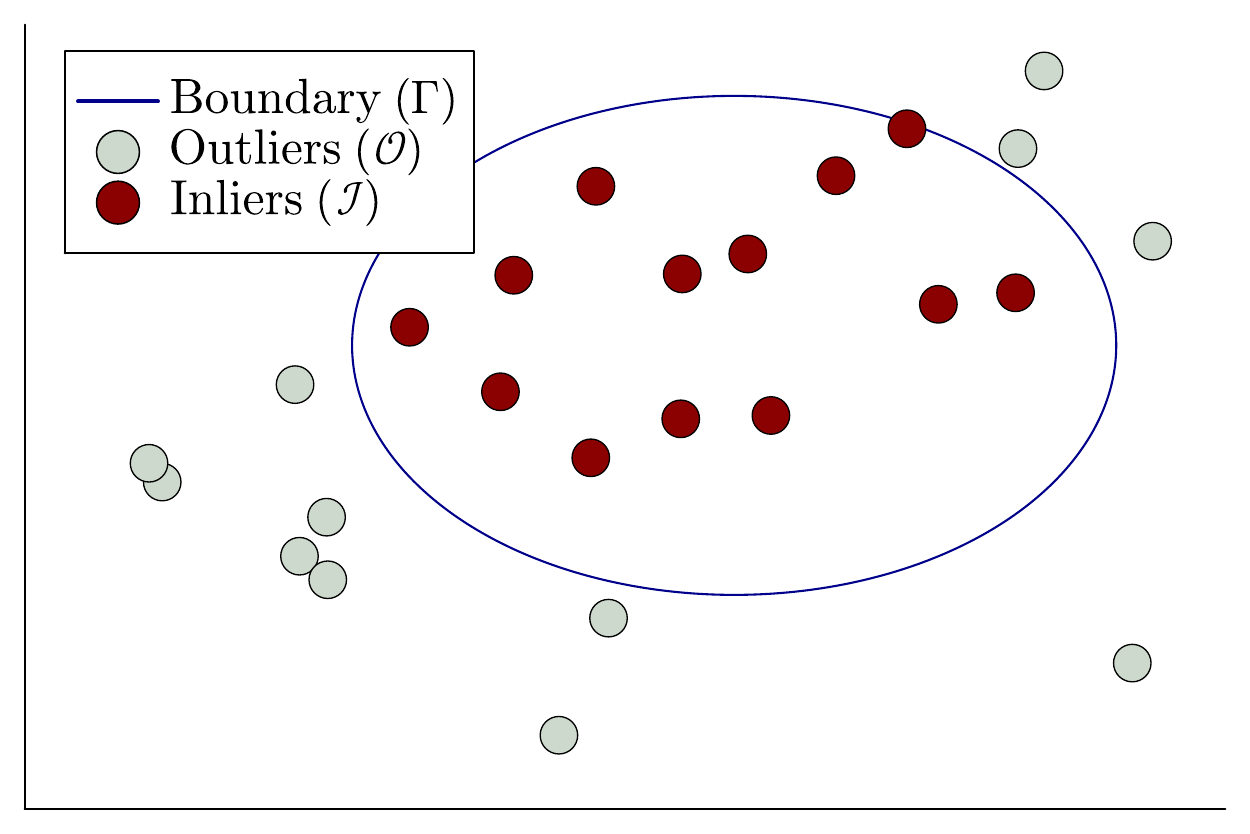}
        \caption{Total variation \eqref{eq:TVO}}
        \label{fig:tv}
    \end{subfigure}
    \begin{subfigure}[b]{0.3\textwidth}
        \centering
        \includegraphics[width=\textwidth]{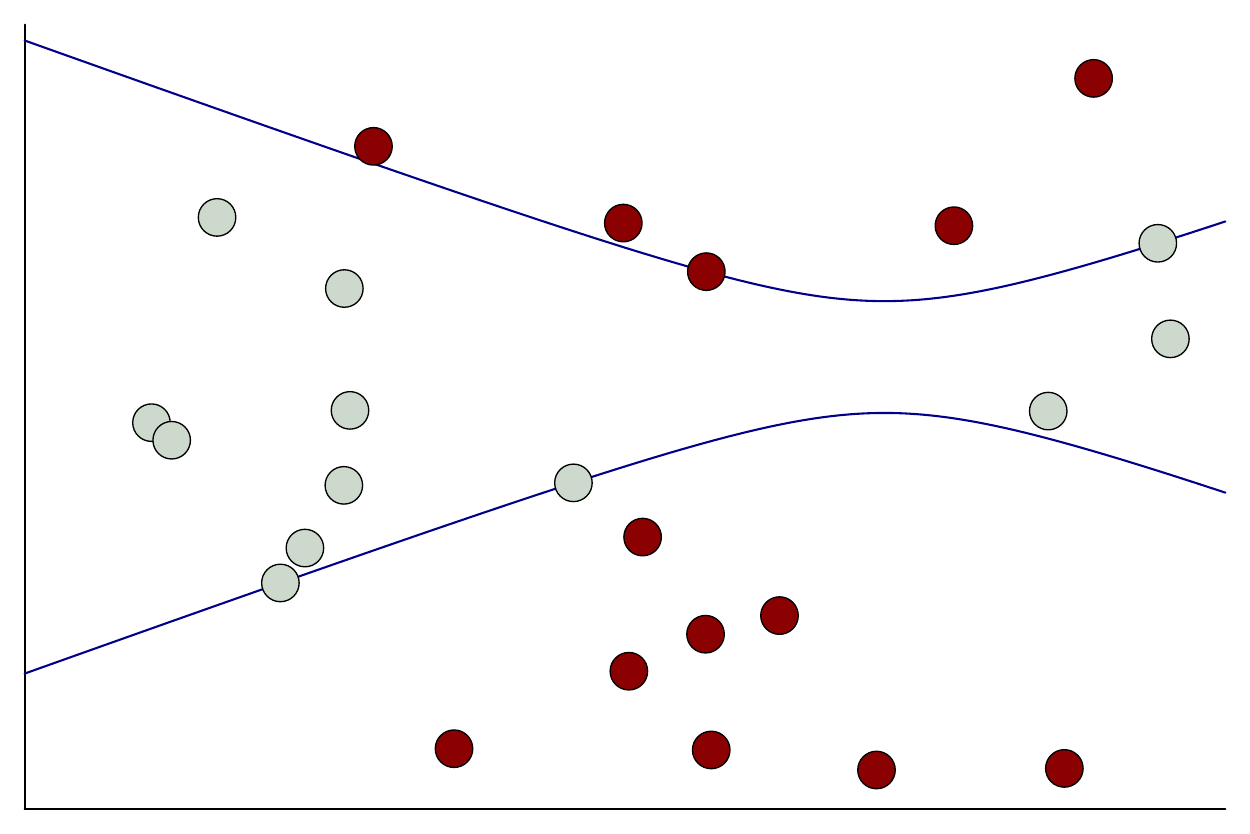}
        \caption{Difference of variance \eqref{eq:DVO}}
        \label{fig:dov}
    \end{subfigure}

    \begin{subfigure}[b]{0.3\textwidth}
        \centering
        \includegraphics[width=\textwidth]{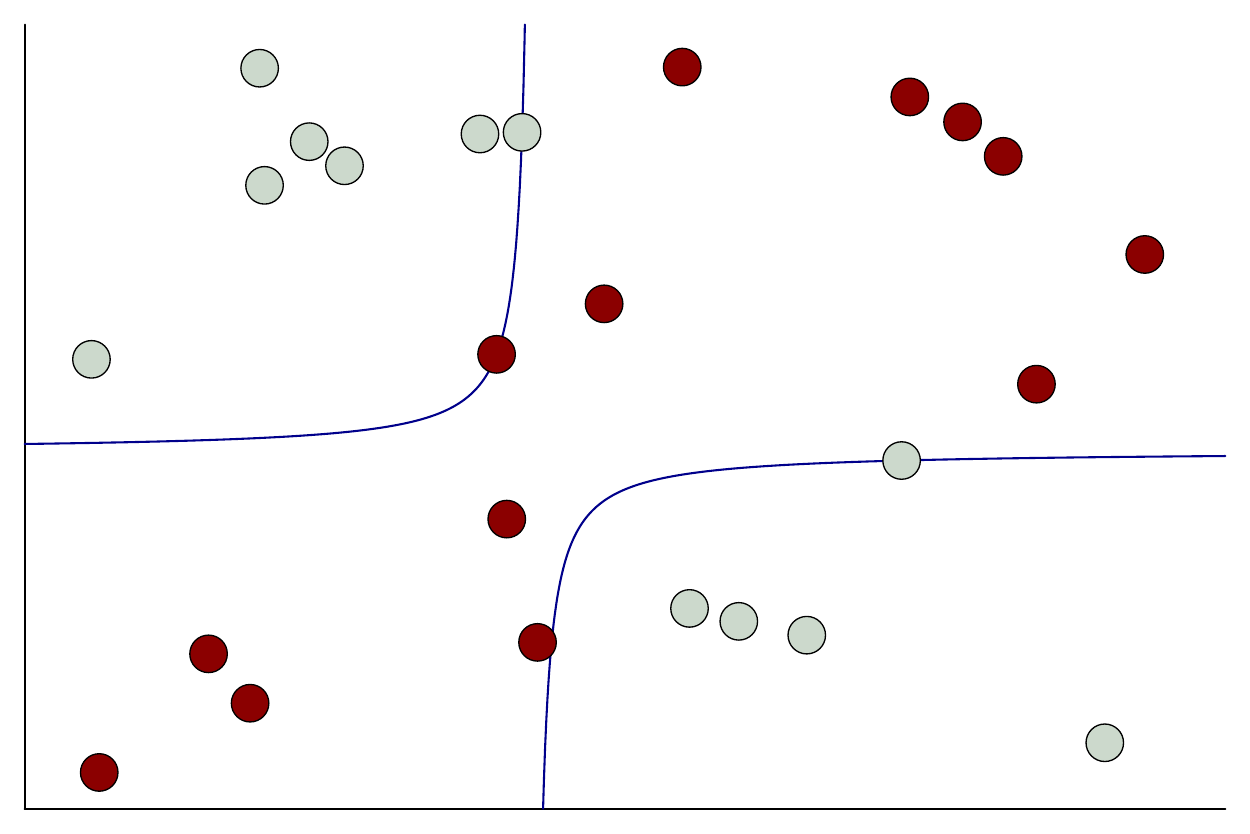}
        \caption{Covariance 
        \eqref{eq:CVO}}
        \label{fig:cov}
    \end{subfigure}
    \begin{subfigure}[b]{0.3\textwidth}
        \centering
        \includegraphics[width=\textwidth]{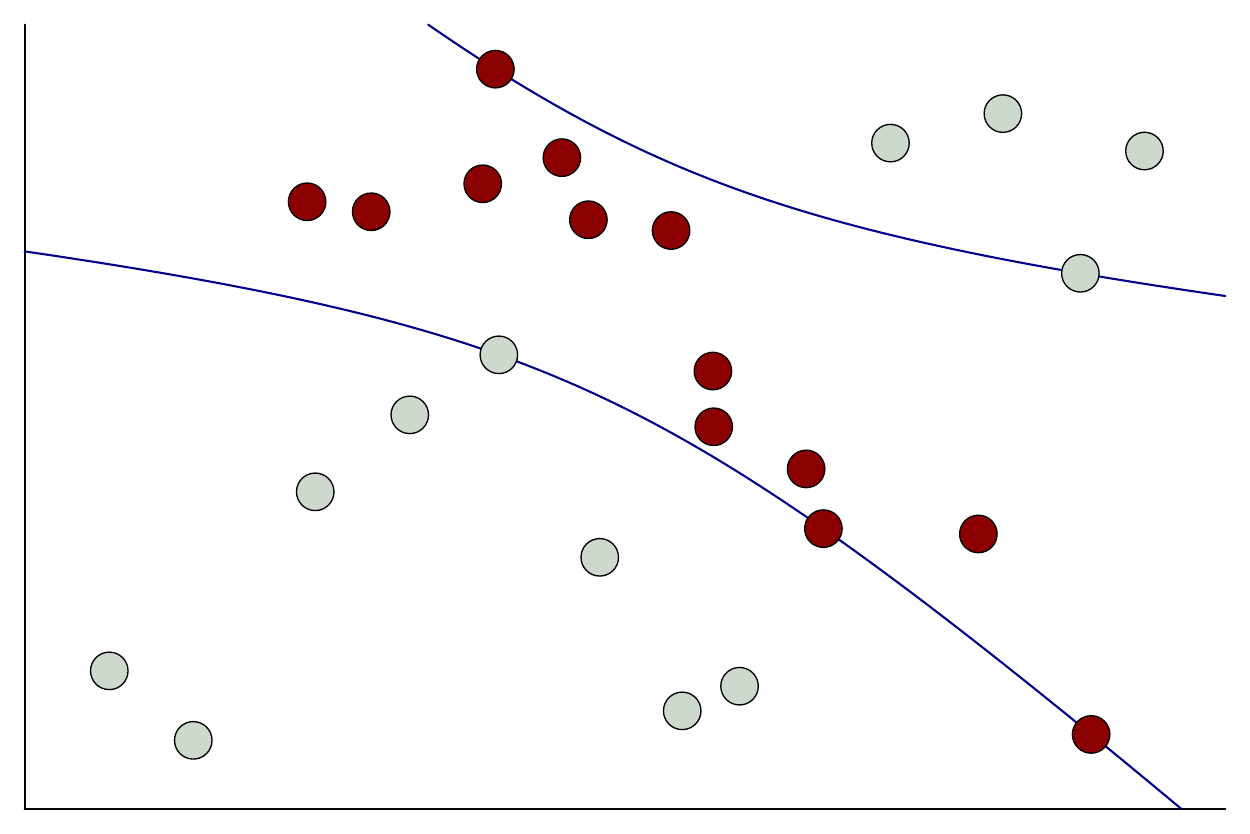}
        \caption{Coefficient of determination \eqref{eq:CDO}}
        \label{fig:cor}
    \end{subfigure}
    
    \caption{The quadratic decision boundaries $\Gamma$ for each objective function.}
    \label{fig:four_panel}
\end{figure}

\section{Experiments}
\label{sec:exp}

The following section provides experimental evidence for our claims. Scripts were written in Julia \cite{bezanson2017julia} with the JuMP \cite{Lubin2023} application program interface (API) for the Ipopt \cite{wachter2006implementation} optimizer. Each dataset for each trial was generated by the Mersenne Twister \cite{matsumoto1998mersenne}, seeded by an 128-bit binary string in turn produced by a second, primary random number generator. The primary generator also consisted of the Mersenne Twister and was seeded with the bitstring corresponding to 123 in decimal. Both $X$ and $Y$ were generated along independent uniform distributions on $\left[0, 1 \right]$ such that expected value of $R^2$ is 0. As to the exact algorithm for computing quadratic sweeps, the less efficient variant \textproc{NaiveQuadraticSweep} was employed. All experiments are fully reproducible and located at \url{https://github.com/marc-harary/QuadraticSweep}. As to hardware, all scripts were run on a Google Cloud Computing virtual machine with an AMD EPYC 7B13 processor; 1 core per socket, 2 threads per core, and 2 CPUs in total; and 8GB of RAM.

\subsection{Separability}

To empirically verify Claim \ref{claim:sep} and the other intuitions developed in Section \ref{sec:sep}, we attempted to linearly separate $\mathcal I$ and $\mathcal O$ after lifting the data by $\mathcal L_4$ and $\mathcal L_5$ for each objective function. If our reasoning is correct, we expect linear separability to be possible only when the lifting transform matches the analytic structure of the objective function. All five objectives should be linearly separable under $\mathcal L_5$; however, only $\sigma^2_x + \sigma^2_y$, $\sigma^2_x - \sigma^2_y$, and $\sigma_{xy}$ should be separable under $\mathcal L_4$.

Separation under each pair of transforms and objectives was attempted for $N=1,000$ trials by applying interior point methods \cite{wright1997primal} to \eqref{eq:HDO}. We partitioned $\mathcal D$ into $\mathcal I$ and $\mathcal O$ by brute-force combinatorial search prior to lifting in order to obtain the ground-truth solution. As to the sizes of the dataset, we let $k=10$ and $n=20$ for these broad initial tests. The distance $d^*$ was assumed to be non-zero within a tolerance $\epsilon = 1 \times 10^{-10}$. 

Results for each pair of transform and objective are tabulated in Table \ref{tab:sep}, including the rate of success at which $\mathbf V^{(\mathcal I)}$ and $\mathbf V^{(\mathcal O)}$ were linearly separated and the minimal primal and dual objective values across all trials. Combinations for which separability was not obtained with perfect accuracy are in bold. Our findings align with our predictions, with separability being impossible only for $r^2$ and $R^2$ under $\mathcal L_4$.

\begin{table*}
  \caption{Minimum values for quadratic separability via Ipopt for $N=1,000$ trials by objective and lift functions}
  \centering
  \begin{tabular}{lccccc}
    \toprule
    Objective     & Transform  & Success rate & Min. primal & Min. dual \\
    \midrule
    \multirow{2}{*}{$r$} & $\mathcal L_5$ & 1.000 & $8.54 \times 10^{-8}$ & $5.99 \times 10^{-8}$ \\
                                    & $\boldsymbol{\mathcal L_4}$ & \textbf{0.326} & $\mathbf{-2.46 \times 10^{-16}}$ & $\mathbf{-8.13 \times 10^{-8}}$ \\
    \multirow{2}{*}{$R^2$} & $\mathcal L_5$ & 1.000 & $1.78 \times 10^{-6}$ & $3.52 \times 10^{-6}$ \\
                                    & $\boldsymbol{\mathcal L_4}$ & \textbf{0.216} & $\mathbf{1.86 \times 10^{-16}}$ & $\mathbf{-8.80 \times 10^{-8}}$  \\
    \multirow{2}{*}{$\sigma_{XY}$} & $\mathcal L_5$ & 1.000 & $3.64 \times 10^{-6}$ & $7.23 \times 10^{-6}$ \\
                                    & $\mathcal L_4$ & 1.000 & $1.16 \times 10^{-7}$ & $1.70 \times 10^{-7}$  \\
    \multirow{2}{*}{$\sigma^2_X + \sigma^2_Y$} & $\mathcal L_5$ & 1.000 & $1.14 \times 10^{-5}$ & $2.27 \times 10^{-5}$  \\
                                    & $\mathcal L_4$ & 1.000 & $2.45 \times 10^{-5}$ & $4.90 \times 10^{-5}$  \\
    \multirow{2}{*}{$\sigma^2_X - \sigma^2_Y$} & $\mathcal L_5$ & 1.000 & $1.24 \times 10^{-6}$ & $2.43 \times 10^{-6}$ \\
                                    & $\mathcal L_4$ & 1.000 & $1.75 \times 10^{-6}$ & $3.47 \times 10^{-6}$  \\
    \bottomrule
  \end{tabular}
  \label{tab:sep}
\end{table*}

\subsection{Optimality}
We next sought to demonstrate the correctness of the quadratic sweep itself, while emphasizing the statistical limitations of a combinatorial approach to outlier detection. Comparisons were made to simulated annealing \cite{kirkpatrick1983optimization}, least trimmed squares \cite{rousseeuw1984least}, RANSAC \cite{fischler1981random}, and the Theil-Sen estimator \cite{theil1950rank}. For the first, we employed the simanneal \cite{simanneal} package for Python with a maximum temperature ($T_\text{max}$) of 100,000, minimum temperature ($T_\text{min}$) of 1, and maximum of 10,000 steps \cite{kirkpatrick1983optimization}. Though efficient algorithms for LTS exist \cite{mount2014least}, we simply applied brute-force combinatorial search to identify the subset with the smallest sum of squares. RANSAC, run with the default value of 100 steps and a minimum number of $k$ inliers, and the Theil-Sen estimator were implemented via Scikit-Learn \cite{pedregosa2011scikit}. 

Each algorithm was test for 1,000 trials for $n \in \left\{15, 20, 25 \right\}$, except for the quadratic sweep, which was tested for a total of 1,000,000, 100,000, and 10,000 trials, respectively. In each instance, $k$ was set to $\lceil n / 2 \rceil$. The first three methods were benchmarked both by computing the success rate as defined above and the ratio of the $R^2$ values of the ground-truth and predicted sets. Because RANSAC is not guaranteed to predict exactly $k$ inliers and computing the Theil-Sen estimator does not return any whatsoever, the success rates were not computed. Because the expected value of $R^2$ is 0, a useful robust estimator like Theil-Sen should thus return a similar result. On the other hand, at large $n$ and small $k$, it should be possible to find a $\mathcal I \subset \mathcal D$ with a non-zero $R^2$. In other words, if our theory is correct, $\mathbb R^2$, as we have defined it, could very well be greater than 0. 

\begin{figure}[b!]
    \centering
    \includegraphics[width=0.9\linewidth]{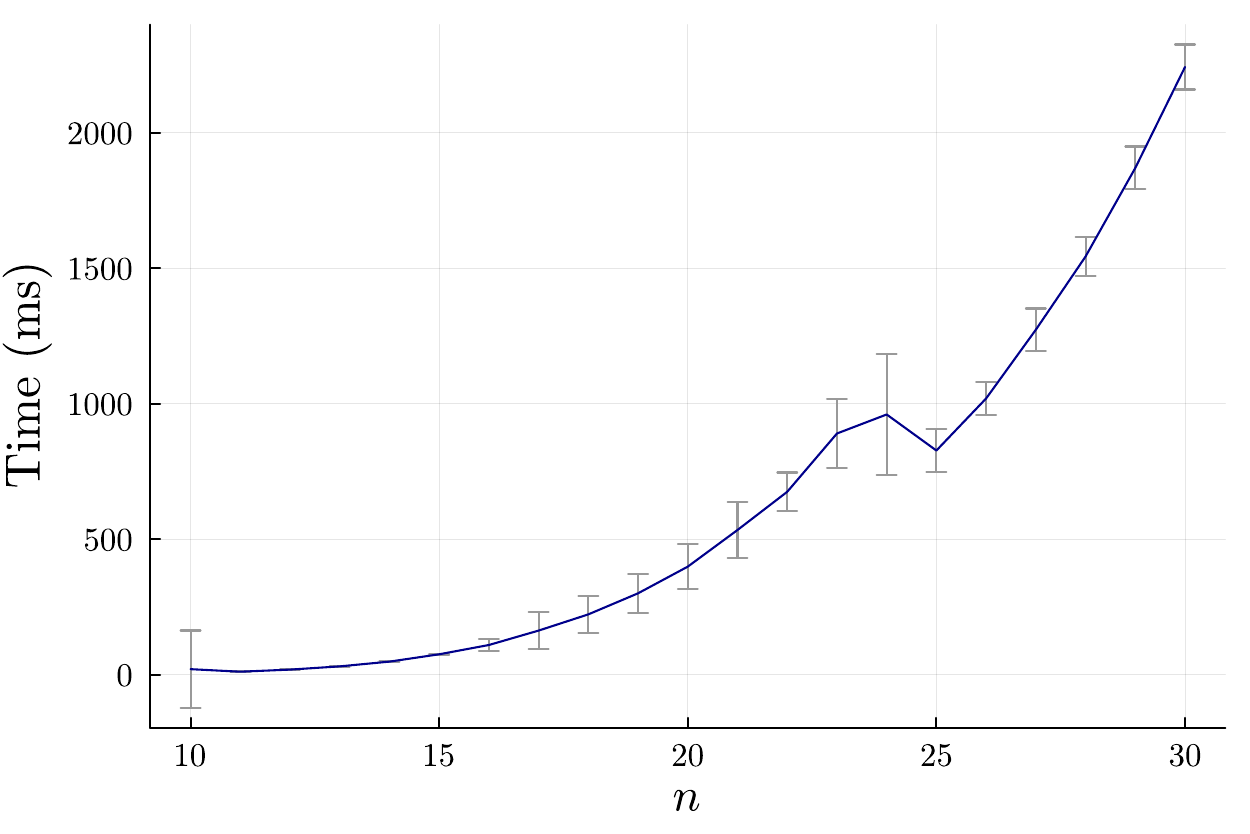}
    \caption{Trial time for the quadratic sweep versus the size of the dataset $n$ for $N=100$ trials}
    \label{fig:time}
\end{figure}

The quadratic sweep (Table \ref{tab:sims}) exhibited a 100\% success rate for the over 1,000,000 trials to which it was subjected, also obtaining a mean $R^2$ ratio of 1.000. Simulated annealing obtained a high success rate of 80.3\% at $n=15$, though performance dropped considerably for higher values of $n$. However, it still exhibited a relatively high $R^2$ ratio even at $n=25$ of 80.9\%. LTS likewise demonstrated performance that decreased monotonically with $n$ to a success rate of 19\% and $R^2$ ratio of 63.5\% at $n=25$. Interestingly, RANSAC's performance actually increased with $n$, although its $R^2$ ratio remained below 0 for all sizes. The Theil-Sen estimator obtained positive $R^2$ ratios, although they were quite close to 0.

The experimental findings hence recapitulate our mathematical results. They also highlight the pitfalls of a combinatorial algorithm for $R^2$ when $n$ is high and $k$ is low. Misapplying our method, we are able to ``cherry-pick'' a small subset with a high degree of collinearity despite the fact that the overall distribution contains no such trend.

\begin{table}
  \caption{Success rate and $R^2$ ratios by method and $n$ ($N=1,000$ trials)}
  \label{sample-table}
  \centering
  \begin{tabular}{lccr}
    \toprule
    Method & $n$ & Success rate & $R^2$ ratio \\
    \midrule
    \multirow{3}{*}{Quadratic sweep} & \bf 15\textsuperscript{*} & \bf 1.000 & \bf 1.000 \\
                                    & \bf 20\textsuperscript{†} & \bf 1.000 & \bf 1.000 \\
                                    & \bf 25\textsuperscript{\ddag} & \bf 1.000 & \bf 1.000 \\
    \multirow{3}{*}{Simulated annealing} & 15 & 0.803 & 0.989 \\
                                     & 20 & 0.062 & 0.912 \\
                                     & 25 & 0.002 & 0.809 \\
    \multirow{3}{*}{LTS} & 15 & 0.313 & 0.701 \\
                                     & 20 & 0.239 & 0.661 \\
                                     & 25 & 0.19 & 0.635 \\
    \multirow{3}{*}{RANSAC} & 15 & - & -0.332 \\
                                     & 20 & - & -0.286 \\
                                     & 25 & - & -0.262 \\
    \multirow{3}{*}{Theil-Sen} & 15 & - & 0.0190 \\
                                     & 20 & - & 0.0170 \\
                                     & 25 & - & 0.0138 \\
    \bottomrule
    \addlinespace[0.5em]
    \multicolumn{4}{l}{\textsuperscript{*} $N=1,000,000$} \\
    \multicolumn{4}{l}{\textsuperscript{†} $N=100,000$} \\
    \multicolumn{4}{l}{\textsuperscript{\ddag} $N=10,000$.}
  \end{tabular}
  \label{tab:sims}
\end{table}

\subsection{Performance}
We finally tested our implementation of \textproc{NaiveQuadraticSweep} in Julia to determine the relationship between trial time and dataset size. Sizes varied such that $n \in \left[10..30 \right]$ and once more we set $k = \lceil n / 2 \rceil$. Again, all tests succeeded in identifying the ground-truth value $\mathcal I$ as computed by brute-force combinatorial search. The mean time varied from 20.2 milliseconds (ms) at $n=10$ to 2,242.9 ms at $n=30$ (Table \ref{fig:time}), indicating computational efficiency.

\section{Conclusion}
\label{sec:concl}

In this work, we consider combinatorially optimizing the coefficient of determination $\left( R^2 \right)$ of a bivariate dataset $\mathcal D$. Called the \textit{quadratic sweep}, our method finds the subset $\mathcal O$ whose exclusion results in the greatest $R^2$. At the core of our algorithm is the observation that $\mathcal O$ and $\mathcal I$ are separable in $\mathbb R^2$ by a quadratic decision boundary $\Gamma$. Principal component analysis (PCA) \cite{hotelling1933analysis}, several simpler problems, and extensive experimental evidence motivate this conjecture. Search for the decision boundary can be performed by lifting the data to a higher-dimensional embedding space in $\mathbb R^5$ in which $\mathcal I$ and $\mathcal O$ become linearly separable. An efficient topological sweep \cite{edelsbrunner1986topologically, edelsbrunner1987algorithms} of $\mathbb R^5$ therefore searches for $\Gamma$, and, by extension, $\mathcal I$. While quadratic separability remains formally unproven, millions of trials furnish cogent experimental evidence of our algorithm and theorems' correctness.

A key limitation of our work is that fact that combinatorial optimization can easily mislead the statistician regarding the presence of a linear trend. If $k$ is small relative to $n$, quadratic sweeps can easily overfit to subsets with a large and unrepresentative $R^2$. In this respect, our approach can easily be too efficient for real-world applications if not carefully employed for a small number of outliers.

Future work should naturally aim to complete the proofs in Section \ref{sec:sep}. As mentioned in Section \ref{sec:sep}, we hypothesize that matroid filtrations and level sets might be useful for proofs by induction \cite{oxley2006matroid}, given the fact that similar proofs can be found in matroid optimization \cite{hochbaum1983efficient}. Improvements to the algorithm itself might also be made to further improve asymptotic runtime, perhaps by placing code to dynamically update $\mathcal I$ within the topological sweep of \cite{clarkson1988applications}; in a related fashion, dynamic convex hulls \cite{brodal2002dynamic} might otherwise be helpful in rendering the sweep more efficient. More practical implementations of quadratic sweeps might also eschew exactness, as geometric approximations of separability \cite{matouvsek1991approximate} often result in vastly shorter runtimes. Rather than enforcing strict separability, perhaps creating ``slabs'' of a certain width $\delta$ could provide valuable relaxations of the problem to render it more easily solvable. Similar approaches have been beneficial for support vector machines (SVMs) with soft margins \cite{kecman2001learning} and convex hull approximation \cite{chan2001dynamic, bentley1982approximation}.

Finally, generalizing quadratic separability to multivariate models may yield valuable insights into higher-dimensional metrics of goodness of fit. Such generalizations could amount to quadric separability, which, in turn, might offer an opportunity for similar lifting transforms and topological sweeps. Applications of our method to simpler objective functions in Section \ref{sec:sep} is also meant to illustrate the subtle importance of quadric surfaces to a wide variety of statistical estimators. Similar techniques may hold significant promise for geometric approaches to robust statistics in other contexts.

\bibliographystyle{plain}
\bibliography{references}

\end{document}